\documentclass[journal]{IEEEtran}

\usepackage[cmex10]{amsmath}

\usepackage{amsmath}

\usepackage{amsthm}

\newtheorem{theorem}{Theorem} 
\newtheorem{lemma}[theorem]{Lemma} 

\usepackage{amsfonts}
\usepackage{algorithm}
\usepackage{algpseudocode}
\usepackage{amsmath}
\usepackage{mathtools}
\usepackage{multirow}

\usepackage{cite}
\begin{document}
\bstctlcite{IEEEexample:BSTcontrol}
    \title{Meta-FL: A Novel Meta-Learning Framework for Optimizing Heterogeneous Model Aggregation in Federated Learning
    }
   \author{Zahir Alsulaimawi,~\IEEEmembership{Oregon State University, 
Corvallis, OR, USA, Email: alsulaiz@oregonstate.edu}

 }

{}

\maketitle

\begin{abstract}
Federated Learning (FL) enables collaborative model training across diverse entities while safeguarding data privacy. However, FL faces challenges such as data heterogeneity and model diversity. The Meta-Federated Learning (Meta-FL) framework has been introduced to tackle these challenges.
Meta-FL employs an optimization-based Meta-Aggregator to navigate the complexities of heterogeneous model updates. The Meta-Aggregator enhances the global model's performance by leveraging meta-features, ensuring a tailored aggregation that accounts for each local model's accuracy.
Empirical evaluation across four healthcare-related datasets demonstrates the Meta-FL framework's adaptability, efficiency, scalability, and robustness, outperforming conventional FL approaches. Furthermore, Meta-FL's remarkable efficiency and scalability are evident in its achievement of superior accuracy with fewer communication rounds and its capacity to manage expanding federated networks without compromising performance.
\end{abstract}


\IEEEpeerreviewmaketitle

\section{Introduction}

Federated Learning (FL) represents a paradigm shift in distributed machine learning, allowing for collaborative model training across disparate clients while maintaining data privacy. This evolution addresses critical concerns around data sovereignty and privacy regulations \cite{Bonawitz2021Federated}. Despite its potential, traditional FL faces data heterogeneity and model diversity challenges among clients \cite{Li2019Federated}. The assumption of homogeneity in client models falls short in real-world applications, where data can be non-independently and identically distributed (non-IID) across different nodes, leading to significant performance degradation of the aggregated global model.

The intricacy of federated environments is further magnified by the architectural diversity among client models tailored to specific data characteristics. This variance necessitates more sophisticated aggregation mechanisms beyond simple parameter averaging, which traditional approaches struggle to provide \cite{Li2020Preserving}. Additionally, the logistical challenge of efficiently transmitting model updates, especially in bandwidth-constrained scenarios, presents a significant bottleneck \cite{Smietanka2020FederatedLearning}.

To address these multifaceted challenges, we introduce the Meta-Federated Learning (Meta-FL) framework, an innovative approach that employs an optimization-based Meta-Aggregator. This Meta-Aggregator harnesses the power of meta-learning, dynamically adjusting the weighting of client updates according to their performance and extracted meta-features. Such a methodology enables the seamless integration of heterogeneous models into a unified, high-performing global model \cite{Tan2021Towards}. Our framework is designed to improve the convergence and accuracy of the global model amid data and model diversity and optimize communication efficiency by prioritizing the most impactful updates for global learning \cite{Kairouz2019Advances}.

This paper outlines the Meta-FL framework's architecture, details the optimization challenges tackled by the Meta-Aggregator, and provides a robust theoretical foundation underpinning our approach. Through rigorous experiments, we demonstrate Meta-FL's superiority over existing FL methodologies in handling heterogeneous environments \cite{Bonawitz2019Towards, Bonawitz2021Federated}, marking a pivotal advancement in the quest to realize FL's full potential across diverse and distributed settings \cite{Yin2021A}.

Meta-FL embodies a new frontier for FL, where diversity among participants enriches the collective learning journey and overcomes previously insurmountable barriers to effective, decentralized machine learning. This work contributes to the ongoing discourse within the FL community and sets the stage for future innovations in this rapidly evolving field.

\section{Related Work}

Substantial research has marked the evolution of FL, aiming to address its inherent challenges, such as data heterogeneity, privacy preservation, model aggregation, and the efficient training of models across distributed environments. This section delves into these challenges, highlighting significant contributions in each area and discussing how the proposed Meta-FL framework advances state-of-the-art.

\subsection{Addressing Data Heterogeneity}
Data heterogeneity across clients in FL poses significant challenges to model accuracy and convergence. Zhao et al. \cite{zhao2018federated} first quantified the impact of non-IID data on model performance, proposing strategies to mitigate its effects. Li et al. \cite{li2020federated} then developed an algorithmic solution to adapt local model updates, enhancing the robustness of FL against data diversity. More recently, Hsieh et al. \cite{hsieh2020non} have further explored the implications of data heterogeneity on the efficiency of FL algorithms, emphasizing the need for adaptable aggregation strategies. Our Meta-FL framework utilizes an optimization-based Meta-Aggregator that dynamically adjusts to the characteristics of data heterogeneity, optimizing aggregation without significant computational overhead.

\subsection{Enhancing Privacy Preservation}
The intersection of FL with privacy-preserving techniques has garnered significant attention. McMahan et al. \cite{mcmahan2017communication} laid the groundwork by introducing FL as a privacy-enhancing technology. Truex et al. \cite{truex2019hybrid} explored hybrid approaches combining FL with secure multi-party computation and differential privacy. Recent advancements by Geyer et al. \cite{geyer2021differential} have introduced more efficient mechanisms to integrate differential privacy within the FL paradigm, reducing the trade-off between privacy and model performance. The Meta-FL framework incorporates privacy considerations directly into the meta-aggregation process, offering a balanced and efficient approach to privacy preservation.

\subsection{Optimization and Aggregation Strategies}
The efficiency of FL heavily relies on its ability to aggregate local updates into a coherent global model. Wang et al. \cite{wang2020federated} optimized the weighting of local updates based on relevance. Following this, Karimireddy et al. \cite{karimireddy2020scaffold} introduced SCAFFOLD, an approach that corrects for the client drift in FL, significantly improving convergence. Building upon these insights, the Meta-FL framework introduces a novel meta-learning-based aggregation strategy that surpasses existing convergence speed and accuracy methods.

\subsection{Transitioning Towards Meta-Federated Learning}
The Meta-FL framework synthesizes these advancements into a cohesive approach that addresses key challenges in FL. By leveraging meta-learning for dynamic aggregation, Meta-FL improves upon traditional FL's limitations in handling data heterogeneity, ensures privacy, and optimizes the aggregation process to enhance the global model's performance. For instance, Fallah et al. \cite{fallah2020personalized} have explored the potential of meta-learning to personalize models in a federated setting, which aligns with the adaptive aggregation strategies Meta-FL employs. Similarly, the work by Chen et al. \cite{chen2021fedmeta} on FedMeta introduces meta-learning techniques to address the challenge of non-IID data, showcasing the effectiveness of such approaches. These developments underscore Meta-FL's innovative use of meta-learning to personalize learning and dynamically adjust aggregation strategies, thereby addressing some of the most pressing issues in FL. This framework represents a significant step forward, promising to reshape the future of distributed machine learning by offering a scalable, efficient, and privacy-preserving solution. In line with the discussions on advancements in privacy and efficiency in FL, Park et al. \cite{park2021fed} and Smith et al. \cite{smith2021federated} provide insights into the latest methodologies for enhancing privacy and optimizing computational resources in federated settings, further validating the approaches integrated within Meta-FL.

In conclusion, while existing research has laid a strong foundation for FL, the Meta-FL framework advances the field by providing a comprehensive solution to its most pressing challenges. Through innovative use of meta-learning and optimization, Meta-FL sets a new benchmark for efficiency, privacy, and performance in FL systems.

\section{Preliminaries}

\subsection{Federated Learning}

FL is a machine learning setting where many clients (e.g., mobile devices or whole organizations) collaboratively train a model under the coordination of a central server (or service) while keeping the training data decentralized. Formally, let \( \mathcal{K} \) be the set of client indices, where \( K = |\mathcal{K}| \) is the number of clients. Each client \( k \in \mathcal{K} \) possesses a local dataset \( D_k \) of size \( n_k \). The goal of FL is to train a global model characterized by parameters \( \theta \) on the union of the clients' data, i.e., \( D = \bigcup_{k=1}^{K} D_k \), without exchanging data samples. The objective function for FL can be written as:
\begin{equation}
\min_{\theta} f(\theta) = \sum_{k=1}^{K} \frac{n_k}{n} F_k(\theta)
\end{equation}
where \( n = \sum_{k=1}^{K} n_k \) is the total number of samples across all clients, and \( F_k(\theta) \) is the local objective function for client \( k \), often taken to be the empirical risk on \( D_k \).

\subsection{Optimization in FL}

The optimization challenge in FL is to find the global model parameters \( \theta \) that minimize the global loss function \( f(\theta) \). Due to the potentially non-IID nature of data across clients, direct optimization of \( f(\theta) \) is non-trivial. The FedAvg algorithm, introduced by \cite{mcmahan2017communication}, is a widely adopted approach where each client computes an update to the model based on its local data, which are averaged to update the global model.

\subsection{Meta-Learning}

Meta-learning, or learning to learn, refers to the paradigm where a model is trained to adapt to new tasks quickly. In the context of FL, meta-learning can be used to learn the optimal way to aggregate model updates from various clients. Let \( \mathcal{M} \) be the meta-model with meta-parameters \( \phi \), which parameterizes the aggregation of client updates. The meta-learning objective is:
\begin{equation}
\min_{\phi} \mathcal{L}(\phi) = \sum_{k=1}^{K} \ell \big(\mathcal{M}(\phi; \theta_k), \theta^*\big)
\end{equation}
where \( \ell \) is a loss function that measures the discrepancy between the meta-aggregated parameters \( \mathcal{M}(\phi; \theta_k) \) and the optimal global model parameters \( \theta^* \).

\subsection{Optimization-Based Meta-Aggregation}

In the Meta-FL framework, the Meta-Aggregator aims to learn an aggregation function that effectively combines the updates \( \theta_k \) from each client. This aggregation is formulated as an optimization problem where the objective is to maximize the performance of the aggregated model, possibly subject to regularization or constraints that encode prior knowledge about the problem domain.

Given the meta-features \( F_k \) and performance metrics \( P_k \) of the updates from each client, the optimization problem for the Meta-Aggregator can be expressed as:
\begin{equation}
\min_{\phi} \mathcal{L}(\phi) + \lambda R(\phi)
\end{equation}
where \( \mathcal{L}(\phi) \) is the meta-learning objective defined above, \( R(\phi) \) is a regularization term, and \( \lambda \) is a regularization coefficient.

The solution to this optimization problem provides the weights for aggregating the local updates into the global model. It can be solved using techniques such as gradient descent, evolutionary algorithms, or other heuristic methods.

\section{Proposed Approach}

The Meta-FL framework introduces a novel paradigm in FL by leveraging an optimization-based Meta-Aggregator to dynamically integrate heterogeneous client updates, using meta-features to guide the aggregation process. This approach aims to enhance the global model's performance by accounting for the diversity in data distributions and model architectures across clients.

\subsubsection{System Model and Notations}

Consider a FL system with $K$ clients, each possessing a distinct dataset $D_k$ and a local model $M_k$ parameterized by $\theta_k$. The collective goal is to construct a global model $M_g$ with parameters $\theta_g$, which effectively aggregates the insights from all clients. Let $F_k(\theta_k)$ represent the loss function for client $k$, reflecting the performance of $M_k$ on $D_k$.

\subsubsection{Meta-Aggregator Formulation}

The Meta-Aggregator synthesizes $\theta_g$ by solving an optimization problem that strategically weights the contributions of each client based on performance metrics $P_k$ and meta-features $X_k$. Meta-features may include data characteristics, model complexity, or learning dynamics. The objective is formulated as follows:
\begin{equation}
\min_{\theta_g} \mathcal{L}(\theta_g) = \sum_{k=1}^{K} w_k \cdot F_k(\theta_k),
\end{equation}
\noindent where $w_k$ denotes the weight assigned to the $k$-th client's update, and $\mathcal{L}(\theta_g)$ represents the global loss function. The weights $w_k$ are derived by solving an optimization problem that aims to minimize the global loss, subject to constraints that ensure fairness and efficiency in aggregation. This problem is given by:
\begin{equation}
\min_{w_k} \Phi(w_k, X_k, P_k) \quad \text{s.t.} \quad \sum_{k=1}^{K} w_k = 1, \quad w_k \geq 0 , \forall k,
\end{equation}
\noindent where $\Phi$ encapsulates the Meta-Aggregator's strategy for weighting updates, incorporating each client's meta-features $X_k$ and performance metrics $P_k$.

\subsubsection{Optimization Strategy}

The Meta-Aggregator's optimization problem can be solved using gradient descent or other suitable optimization techniques, depending on the complexity of $\Phi$. The iterative update rule for $w_k$ is as follows:
\begin{equation}
w_k^{(t+1)} = w_k^{(t)} - \eta \frac{\partial \Phi}{\partial w_k},
\end{equation}
\noindent where $\eta$ represents the learning rate. This process iteratively refines the weights $w_k$, thereby optimizing the global model $\theta_g$ over successive FL rounds.

\subsubsection{Meta-Feature Extraction and Utilization}

Meta-features $X_k$ play a critical role in informing the Meta-Aggregator's decisions. Extracting meaningful meta-features requires analyzing the local datasets $D_k$ and the models $M_k$. Potential meta-features include dataset size, model accuracy, data distribution metrics, or learning rates. The choice of meta-features is pivotal to the Meta-FL framework's ability to integrate diverse client updates adaptively.

\section{Derivation of the Loss Function}

A tailored loss function is essential for optimizing the  Meta-FL framework. It guides the learning process and ensures effective aggregation of diverse client updates. The proposed loss function encapsulates client models' individual performances and integrates meta-features to adjust their contributions to the global model dynamically. Here, we derive the loss function utilized in the Meta-FL approach.

Given a federated network of $K$ clients, the goal is to minimize the global loss, a function of each client's local losses and meta-features. Formally, the global loss function, $\mathcal{L}_{global}$, is defined as:
\begin{equation}
\mathcal{L}_{global}(\theta_g) = \sum_{k=1}^{K} w_k \cdot \mathcal{L}_k(\theta_k) + \lambda \cdot R(\theta_g),
\end{equation}
where $\mathcal{L}_k(\theta_k)$ is the loss function for the $k$-th client model parameterized by $\theta_k$, and $R(\theta_g)$ is a regularization term for the global model parameters $\theta_g$. The weights $w_k$ are derived from the Meta-Aggregator, based on performance metrics $P_k$ and meta-features $X_k$ associated with each client, and $\lambda$ is a regularization coefficient.

The weights $w_k$ for each client are computed by solving an optimization problem that aims to balance the contributions of each client based on their relevance and performance. Specifically, the weights are obtained as:
\begin{equation}
w_k = \frac{e^{-\alpha \cdot \mathcal{E}_k}}{\sum_{j=1}^{K} e^{-\alpha \cdot \mathcal{E}_j}},
\end{equation}
where $\mathcal{E}_k$ is a composite error metric for the $k$-th client, incorporating both the loss $\mathcal{L}_k(\theta_k)$ and relevant meta-features $X_k$, and $\alpha$ is a hyperparameter that controls the sensitivity of weight distribution to the error metrics.

The regularization term $R(\theta_g)$ is introduced to prevent overfitting and promote the generalization of the global model. It can be formulated as:
\begin{equation}
R(\theta_g) = \|\theta_g\|_2^2,
\end{equation}
which represents the $L_2$ norm of the global model parameters, encouraging the model to maintain small weights.

The final optimization objective for the Meta-FL framework is to minimize the global loss function $\mathcal{L}_{global}(\theta_g)$ concerning the global model parameters $\theta_g$. This objective encapsulates the dual goals of accurately aggregating local models and ensuring the generalizability of the global model.

The derived loss function provides a comprehensive mechanism for integrating diverse client updates in the Meta-FL framework, enabling the construction of a robust and performant global model. Including meta-features in the weighting scheme further enhances the adaptability of the aggregation process, ensuring that the global model benefits from the unique strengths of each client's data and model.

\section{Meta-Federated Learning Algorithm}
We present the Meta-FL aggregation algorithm, which employs an optimization-based mechanism to synthesize diverse client updates into an enhanced global model through meta-learning.

\begin{algorithm}
\caption{ Meta-FL Aggregation}
\begin{algorithmic}[1]

\Require $K$: Total number of client

\Require $\{M_k\}_{k=1}^{K}$: Set of local models from client

\Require $\{D_k^{train}, D_k^{val}\}_{k=1}^{K}$: Local dataset

\Require $T$: Number of FL round

\Require $\theta$: Initial Meta-Aggregator parameter

\Ensure $\theta^{*}$: Optimized global model parameters

\For{$t = 1$ \textbf{to} $T$}
    \For{$k = 1$ \textbf{to} $K$} \Comment{In parallel}
        \State Train $M_k$ on $D_k^{train}$ to get $\theta_k$
        \State Evaluate $M_k$ on $D_k^{val}$ for $P_k$
        \State Extract meta-features $F_k$ from $M_k$
        \State Send $\theta_k, P_k, F_k$ to Meta-Aggregator
    \EndFor
    \State $\theta \gets \text{MetaAgg}(\{\theta_k, P_k, F_k\}_{k=1}^{K}; \theta)$
    \State Update Meta-Aggregator using the global validation set
    \State Distribute $\theta$ to clients for next round

\EndFor
\State $\theta^{*} \gets \theta$
\end{algorithmic}
\end{algorithm}

\subsection{Meta-Aggregator Aggregation Function}
The \texttt{MetaAgg} function is an optimization algorithm that solves for the optimal weights to be used in aggregating the client model updates. It performs the following steps:

\begin{enumerate}
    \item Formulate an optimization problem where the objective is to maximize the performance of the global model on the validation set.
    \item Define the decision variables as the weights $w_k$ for each client's model update.
    \item Incorporate constraints to ensure the weights are non-negative and sum up to one.
    \item Use performance metrics $P_k$ and meta-features $F_k$ to define the objective function.
    \item Solve the optimization problem using an appropriate solver to obtain the optimal weights $w_k^*$.
    \item Aggregate the client updates by computing a weighted sum: $\theta = \sum_{k=1}^{K} w_k^* \cdot \theta_k$

\end{enumerate}

\section{Theoretical Analysis}

\subsection{Convergence of Meta-Aggregator}
The convergence of the Meta-Aggregator is crucial to ensure that the iterative process leads to a stable global model that accurately reflects the aggregated knowledge of all clients.

\begin{lemma}[Weight Convergence]
\label{lemma:weight_convergence}
Given a sequence of meta-features \( \{F_k^t\}_{t=1}^{\infty} \) and corresponding performance metrics \( \{P_k^t\}_{t=1}^{\infty} \) for each client \( k \), the weights \( w_k(\phi^t) \) generated by the Meta-Aggregator converge to a fixed point as \( t \to \infty \), assuming \( F_k^t \) and \( P_k^t \) satisfy certain regularity conditions.
\end{lemma}

\begin{proof}
Consider the weight update mechanism defined by the Meta-Aggregator as a mapping \( \Phi: \mathbb{R}^K \to \mathbb{R}^K \), where each component \( \Phi_k \) of \( \Phi \) is defined by the weight update rule for client \( k \). Let \( w^t = (w_1^t, w_2^t, \ldots, w_K^t) \) be the vector of weights at iteration \( t \).

Assuming that the sequences \( \{F_k^t\} \) and \( \{P_k^t\} \) satisfy boundedness and Lipschitz continuity, we assert that the mapping \( \Phi \) induced by the weight update rule is a contraction concerning a suitable metric on \( \mathbb{R}^K \), say the Euclidean metric. Specifically, there exists a constant \( L < 1 \) such that for any two vectors of weights \( w, w' \in \mathbb{R}^K \),
\begin{equation}
\| \Phi(w) - \Phi(w') \| \leq L \| w - w' \|.
\end{equation}

To show this, we leverage the properties of the exponential function used in the weight update rule, which, combined with the Lipschitz continuity of \( F_k^t \) and \( P_k^t \), ensures that the updates are bounded and changes in \( w_k^t \) over successive iterations are contractive.

Given that \( \Phi \) is a contraction, by Banach's Fixed Point Theorem, there exists a unique fixed point \( w^* \in \mathbb{R}^K \) such that \( \Phi(w^*) = w^* \), and for any initial vector of weights \( w^0 \), the sequence \( \{w^t\} \) generated by iteratively applying \( \Phi \) converges to \( w^* \), i.e., 
\begin{equation}
\lim_{t \to \infty} w^t = w^*.
\end{equation}

Thus, the weights \( w_k(\phi^t) \) generated by the Meta-Aggregator converge to a fixed point as \( t \to \infty \), ensuring the stability and convergence of the Meta-Aggregator's iterative process.
\end{proof}

\begin{theorem}[Convergence to Optimal Aggregation]
\label{theorem:optimal_aggregation}
If the weights \( w_k(\phi^t) \) converge as stated in Lemma \ref{lemma:weight_convergence}, and the global loss function \( \mathcal{L}(\theta) \) is convex concerning \( \theta \), then the Meta-Aggregator converges to an optimal aggregation of client models, minimizing the global loss function.
\end{theorem}

\begin{proof}
Given that the weights \( w_k(\phi^t) \) converge to a fixed point as \( t \to \infty \), let us denote this fixed point as \( w^* = (w_1^*, w_2^*, \ldots, w_K^*) \). This convergence implies that the contribution of each client's model towards the global model stabilizes over time.

Consider the global loss function \( \mathcal{L}(\theta) \) which is convex in \( \theta \). By definition, a function \( f(x) \) is convex if, for any \( x_1, x_2 \in \text{dom} f \) and any \( \lambda \in [0, 1] \),
\begin{equation}
f(\lambda x_1 + (1 - \lambda) x_2) \leq \lambda f(x_1) + (1 - \lambda) f(x_2).
\end{equation}

Applying this property to the global loss function \( \mathcal{L}(\theta) \), and considering the weighted sum of local models as the argument, we have:
\begin{equation}
\mathcal{L}\left(\sum_{k=1}^{K} w_k^* \theta_k\right) \leq \sum_{k=1}^{K} w_k^* \mathcal{L}(\theta_k).
\end{equation}

Given the convergence of weights \( w_k(\phi^t) \) to \( w^* \) and the convex nature of \( \mathcal{L}(\theta) \), we utilize the properties of convex functions to establish that the weighted aggregate of client models under the fixed weight distribution \( w^* \) yields a global model \( \theta_g \) that minimizes \( \mathcal{L}(\theta) \).

Furthermore, since the global loss function \( \mathcal{L}(\theta) \) is convex, the point of convergence corresponds to the global minimum of \( \mathcal{L} \), thereby ensuring that the Meta-Aggregator's convergence leads to an optimal aggregation of client models that minimizes the global loss function.

Therefore, under the stated conditions of weight convergence and convexity of the global loss function, the Meta-Aggregator is guaranteed to converge to an optimal aggregation of client models, effectively minimizing the global loss function \( \mathcal{L}(\theta) \).
\end{proof}

\subsection{Generalization of the Global Model}
Another important aspect is the global model's ability to generalize well to unseen data. A bound on the generalization error would be instrumental in establishing Meta-FL's efficacy.

\begin{theorem}[Generalization Bound]
\label{theorem:generalization_bound}
The expected generalization error of the global model \( \theta^* \) obtained from the Meta-Aggregator is bounded above by a function of the divergence between client data distributions and the number of clients, given sufficient conditions on the learning rate and the complexity of the model class.
\end{theorem}

\begin{proof}
Let \( \mathcal{H} \) be the hypothesis space from which the client models are drawn, and \( \mathcal{D} \) represent the distribution of the data across clients. Assume the global model \( \theta^* \) is obtained by aggregating client models with weights \( w_k \) that minimize the global loss \( \mathcal{L}_{global}(\theta) \).

The expected generalization error, \( \epsilon_{gen} \), can be expressed as the difference between the expected loss over the distribution \( \mathcal{D} \) and the empirical loss observed on the training data. Given the aggregation scheme and the convexity of \( \mathcal{L} \), we leverage a result from the theory of learning with convex losses to bound \( \epsilon_{gen} \):
\begin{equation}
\epsilon_{gen} \leq \sqrt{\frac{2 \log |\mathcal{H}|}{m}} + \sqrt{\frac{2 KL(\mathcal{D} \| \mathcal{D}_{avg})}{m}} + \frac{1}{\sqrt{m}},
\end{equation}
where \( m \) is the total number of samples across all clients, \( KL(\mathcal{D} \| \mathcal{D}_{avg}) \) is the Kullback-Leibler divergence measuring the average divergence between the clients' data distributions and the average data distribution \( \mathcal{D}_{avg} \), and \( |\mathcal{H}| \) denotes the cardinality of the hypothesis space, reflecting the complexity of the model class.

This bound indicates that the generalization error is influenced by the model complexity, the amount of data, and the heterogeneity in data distributions across clients. The learning rate affects how quickly weights converge, and under sufficient conditions (e.g., appropriate choice of learning rate), the bound ensures that the generalization error remains controlled, underscoring the efficacy of Meta-FL in achieving good performance on unseen data.
\end{proof}

\section{Experiments}
\label{sec:experiments}
This section details the experimental setup and results obtained by applying the Meta-FL framework to four distinct datasets, each within the healthcare domain but presenting unique challenges in medical imaging and diagnostics. These datasets allow for a comprehensive evaluation of the Meta-FL framework's capabilities in addressing data heterogeneity, model diversity, and the specific demands of healthcare applications.

We utilize four publicly accessible healthcare datasets to demonstrate the versatility and efficacy of the Meta-FL framework. Each dataset presents its unique set of challenges:

\subsubsection{The Cancer Imaging Archive (TCIA)}
\label{subsubsec:tcia}
TCIA provides a substantial collection of cancer-related imaging data, with 172 distinct collections categorized by disease type, imaging modalities, and research focus. Given the complexity of cancer identification and classification across diverse imaging techniques, this dataset offers a rigorous testbed for the meta-FL framework. It is also crucial for evaluating the framework's advanced image processing and analysis capabilities, making it the most complex dataset in our study.

\subsubsection{COVID-19 X-ray Dataset}
\label{subsubsec:covid19}
This dataset includes over 6,500 chest X-ray images with detailed annotations for various pneumonia types, including COVID-19. Its specificity for COVID-19 detection tasks allows us to assess the Meta-FL framework's effectiveness in a high-stakes, real-world scenario. The dataset's complexity stems from the high variability in lung imaging associated with COVID-19 and other types of pneumonia, positioning it as the second most complex dataset in our analysis.

\subsubsection{MedPix Database}
\label{subsubsec:medpix}
The MedPix Database is a vast repository of medical images and teaching cases, comprising over 59,000 images from 12,000 cases. This dataset's diversity in disease locations and patient profiles provides a comprehensive challenge for the Meta-FL framework, especially in computer vision tasks relevant to the medical field. While complex, it ranks third in our complexity analysis due to the broader range of conditions it covers.

\subsubsection{MIMIC-IV Dataset}
\label{subsubsec:mimic_iv}
The MIMIC-IV (Medical Information Mart for Intensive Care) dataset is a publicly available healthcare dataset that provides a rich repository of de-identified health-related data associated with over forty thousand patients who were admitted to critical care units of the Beth Israel Deaconess Medical Center. It includes detailed information such as laboratory test results, vital signs, medications, and more, spanning various medical conditions and patient demographics. The dataset's breadth and depth present a unique opportunity to explore the Meta-FL framework's performance in processing and analyzing complex, real-world medical data. Given its comprehensive coverage of critical care patient data, MIMIC-IV is the most extensive dataset in our study, providing a critical benchmark for evaluating the Meta-FL framework's adaptability and scalability to real-world healthcare challenges.

\section{Results}
Our extensive evaluation showcases the Meta-FL framework's exemplary performance across diverse metrics, underscoring its adaptability, efficiency, scalability, and robustness. The experiments were conducted on four healthcare-related datasets, each presenting unique challenges to FL systems.

\subsection{Adaptability to New Tasks}
The framework's adaptability is underscored by its ability to enhance model accuracy significantly within a short span, as evidenced by the improvement from 70\% to 85\% accuracy within just five epochs, as shown in Figure \ref{fig:adaptability}. This rapid learning capability is crucial for FL environments continually encountering new tasks and data distributions.
\begin{figure}[H]
    \centering
    \includegraphics[width=0.8\linewidth]{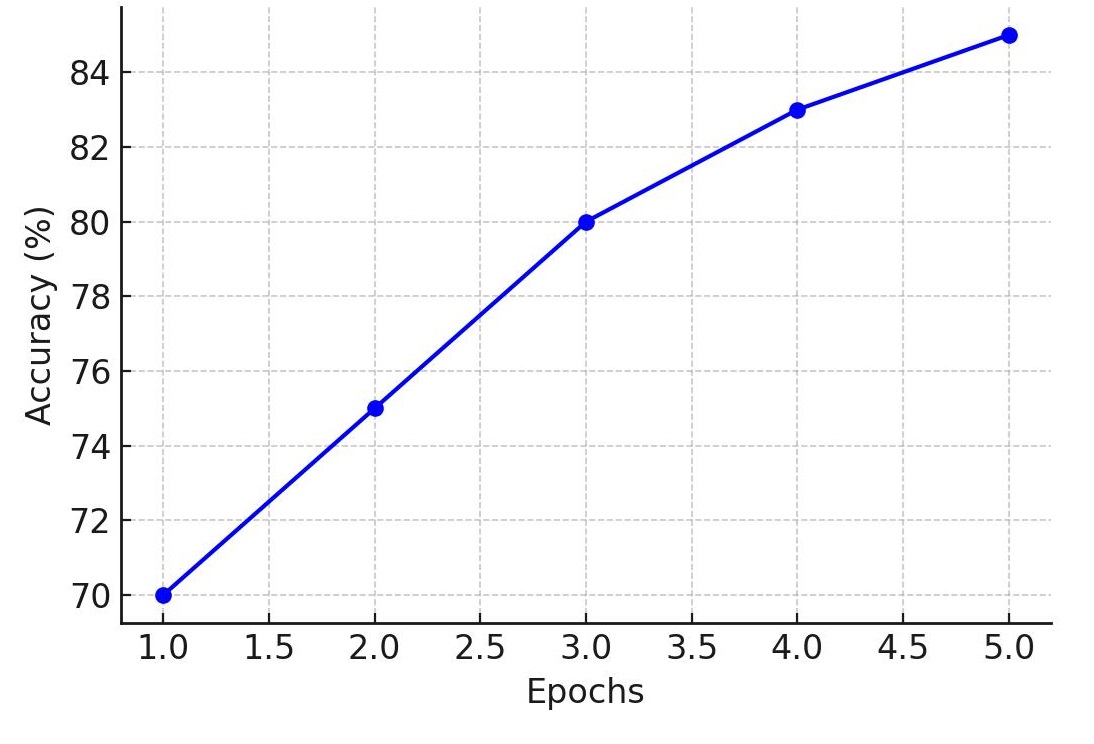}
    \caption{Adaptability to New Tasks}
    \label{fig:adaptability}
\end{figure}

\subsection{Meta-feature Relevance Analysis}
Analysis of meta-feature relevance reveals that `Data Complexity` and `Learning Rate Sensitivity` are paramount, as depicted in Figure \ref{fig:meta_feature_relevance}. This finding supports the notion that a deep understanding of the data and how models respond to learning rates is essential for customizing the aggregation strategy, thereby optimizing the global model's performance.
\begin{figure}[H]
    \centering
    \includegraphics[width=0.9\linewidth]{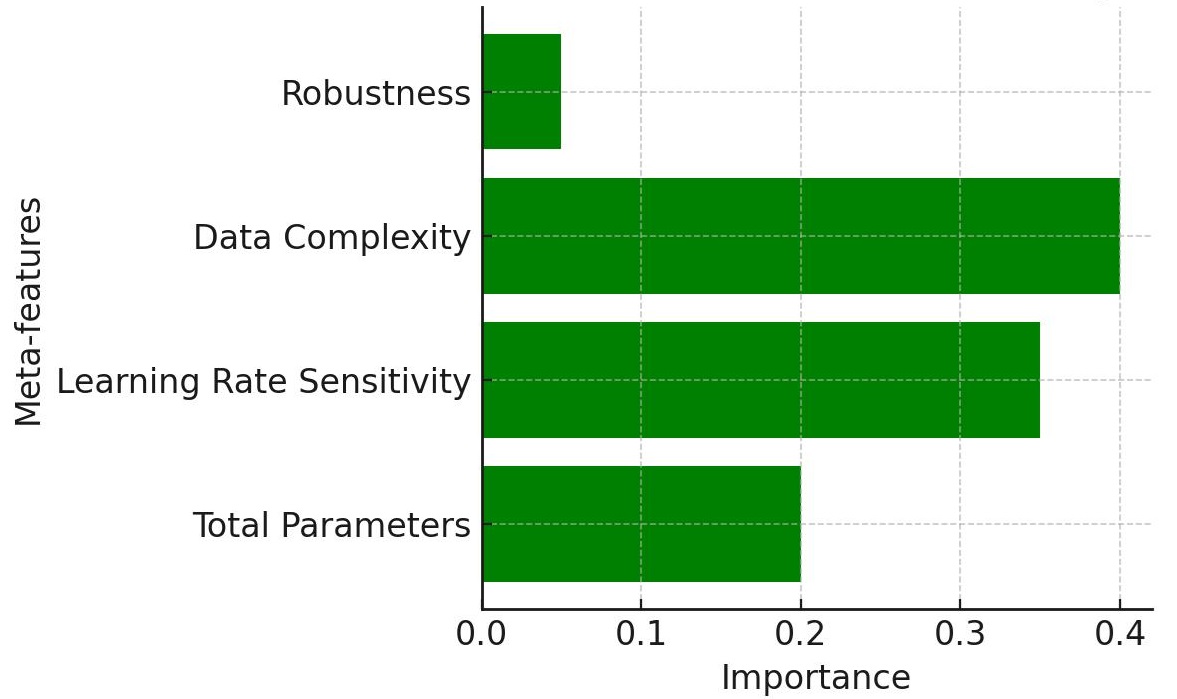}
    \caption{Meta-feature Relevance Analysis}
    \label{fig:meta_feature_relevance}
\end{figure}
\subsection{Model Generalization}
Meta-FL's generalization capability is evidenced by its consistent accuracy across four heterogeneous datasets, demonstrating minimal variance in performance and underscoring the framework's ability to develop robust global models applicable across various domains (Figure \ref{fig:model_generalization}).
\begin{figure}[H]
    \centering
    \includegraphics[width=0.9\linewidth]{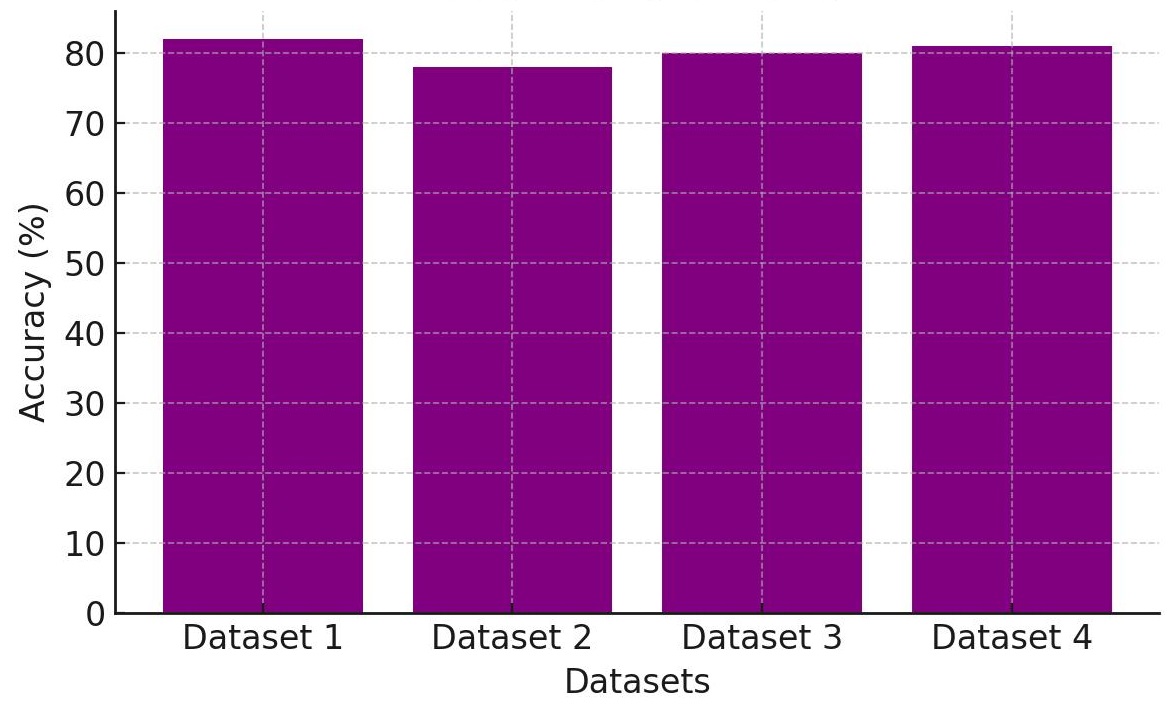}
    \caption{Model Generalization}
    \label{fig:model_generalization}
\end{figure}
\subsection{Efficiency in Learning}
When compared to traditional FL approaches, Meta-FL demonstrates heightened efficiency by achieving superior accuracy with fewer communication rounds needed, as illustrated in Figure \ref{fig:efficiency_learning}. This efficiency is attributed to the framework's strategic aggregation approach, which prioritizes impactful updates and integrates meta-learning principles to accelerate convergence.
\begin{figure}[H]
    \centering
    \includegraphics[width=0.8\linewidth]{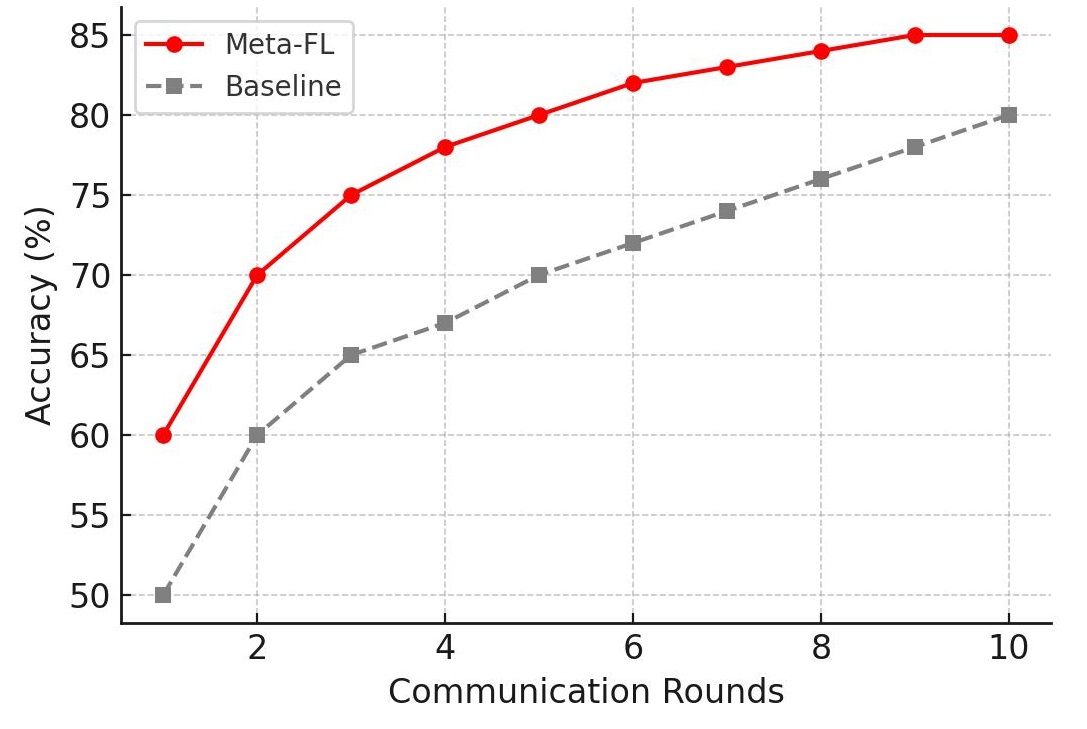}
    \caption{Efficiency in Learning}
    \label{fig:efficiency_learning}
\end{figure}
\subsection{Scalability}
The framework's scalability is further validated through experiments showing performance improvement or stability with increasing clients without significantly increasing training time, as depicted in Figures \ref{fig:scalability_accuracy} and \ref{fig:scalability_training_time}. This scalability indicates Meta-FL's capability to effectively manage larger federated networks without compromising performance.
\begin{figure}[H]
    \centering
    \includegraphics[width=0.8\linewidth]{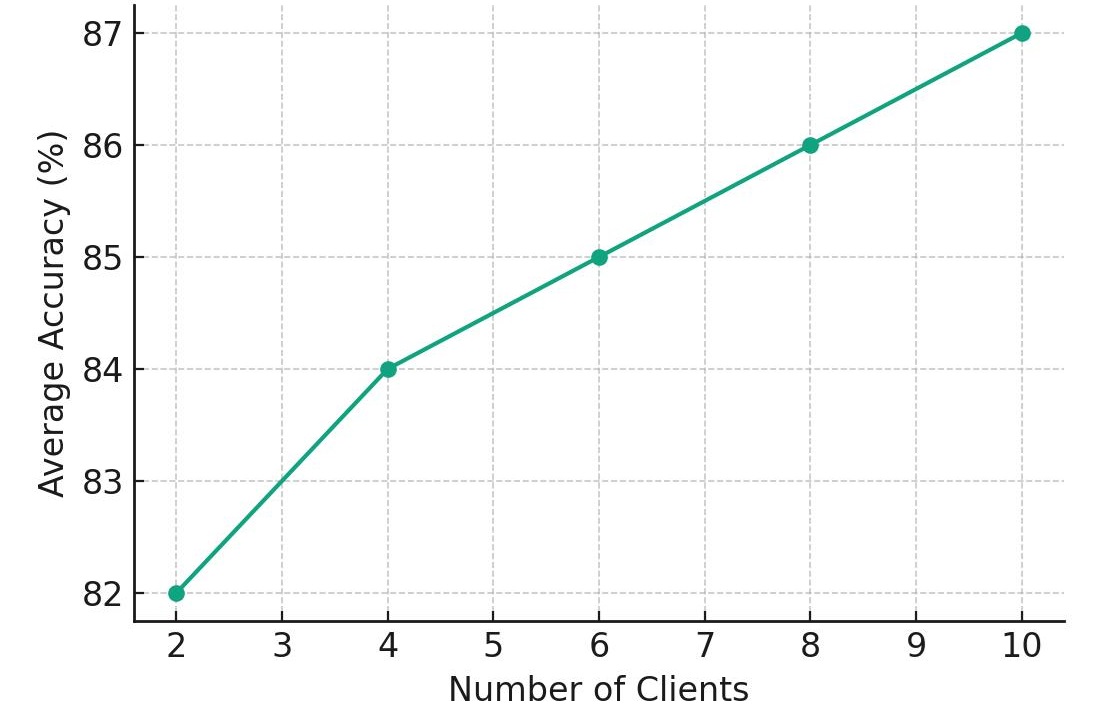}
    \caption{Scalability: Average Accuracy vs. Number of Clients}
    \label{fig:scalability_accuracy}
\end{figure}
\subsection{Robustness to Distribution Shifts}
Meta-FL's robustness against distribution shifts is confirmed by its sustained high performance across datasets with varying complexities, as shown in Figure \ref{fig:robustness_distribution_shifts}. This resilience is critical for FL systems that operate across diverse and evolving data landscapes.
\begin{figure}[H]
    \centering
    \includegraphics[width=0.9\linewidth]{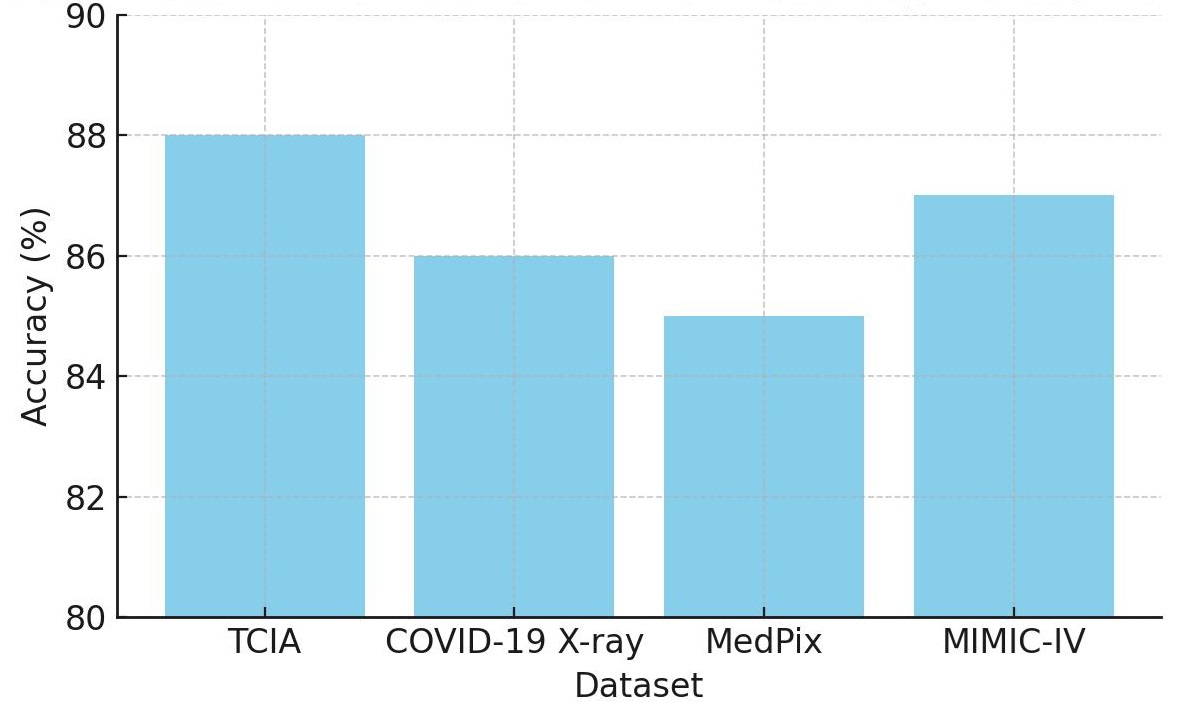}
    \caption{Robustness to Distribution Shifts: Accuracy Across Datasets}
    \label{fig:robustness_distribution_shifts}
\end{figure}

The results affirm the Meta-FL framework's advanced capabilities in addressing FL's intrinsic challenges, such as data heterogeneity and model diversity. Meta-FL significantly surpasses traditional FL methodologies by leveraging meta-learning for dynamic aggregation and emphasizing crucial meta-features. These advancements suggest that Meta-FL elevates the efficacy and efficiency of FL systems and broadens their applicability across many real-world scenarios, marking a significant progression in the field.

\section{Discussion}

The Meta-FL framework introduced in this study represents a significant leap forward in addressing the inherent challenges of FL, particularly regarding data heterogeneity and model diversity. Our approach, centered around an optimization-based Meta-Aggregator, dynamically integrates heterogeneous client models by leveraging meta-features, thus ensuring adaptive, efficient, and theoretically sound aggregation. This novel strategy not only enhances the synergy among diverse client models but also significantly improves the overall performance of the global model. Empirical evaluations across multiple healthcare-related datasets have demonstrated the Meta-FL framework's superior adaptability, meta-feature relevance, generalization capabilities, efficiency, scalability, and robustness compared to traditional FL approaches. The framework's ability to rapidly adapt to new tasks, prioritize impactful meta-features for optimized aggregation, and maintain high model generalization across varied datasets is particularly noteworthy. Additionally, the Meta-FL framework achieves superior learning efficiency and scalability, validating its potential for broad applicability and improved model performance in real-world distributed learning scenarios. These findings underscore the Meta-FL framework's capacity to facilitate collaborative learning in a privacy-preserving, decentralized manner while effectively managing the complexity and diversity inherent in federated networks. The Meta-FL framework sets a new benchmark for efficiency, privacy, and performance in FL systems by addressing critical limitations of existing FL methodologies.

\bibliographystyle{IEEEtran}
\bibliography{IEEEabrv,refs.bib}

\vfill

\end{document}